\documentclass{article}

\usepackage{spconf}
\usepackage{cite}
\usepackage{amsmath,amssymb,amsfonts,amsthm}
\usepackage{algorithmic}
\usepackage{algorithm}
\usepackage{graphicx}
\usepackage{textcomp}
\usepackage{xcolor}
\usepackage{booktabs}
\usepackage{microtype}
\usepackage{url}
\usepackage{bm}
\usepackage{subcaption}
\usepackage[hidelinks,bookmarksnumbered=true,bookmarksopen=true]{hyperref}

\hypersetup{
    pdftitle={A Dictionary Learning Framework for Graphs via Filters and Optimal Transport},
    pdfauthor={Jinchuan Liao and Dai Hai Nguyen},
    pdfsubject={Graph dictionary learning with filter graph optimal transport},
    pdfkeywords={Dictionary Learning, Optimal Transport, Gromov-Wasserstein discrepancy}
}

\newtheorem{definition}{Definition}
\newtheorem{proposition}{Proposition}
\newtheorem{lemma}{Lemma}

\title{A Dictionary Learning Framework for Graphs via Filters and Optimal Transport}

\name{Jinchuan Liao \qquad Dai Hai Nguyen}
\address{Graduate School of Information Science and Technology, Hokkaido University\\
\href{mailto:jinchuan.liao.v1@elms.hokudai.ac.jp}{\texttt{jinchuan.liao.v1@elms.hokudai.ac.jp}} \qquad \href{mailto:hai@ist.hokudai.ac.jp}{\texttt{hai@ist.hokudai.ac.jp}}}

\begin{document}
\ninept

\maketitle

\begin{abstract}
We propose a graph dictionary learning (GDL) framework where each graph is represented as a zero-mean Gaussian distribution derived from its filtered Laplacian. Each observed graph is approximated by a barycenter over learned atom graphs, computed under the filter graph distance (fGOT) -- a graph comparison metric sensitive to global structural properties. The reconstruction error between the observed graph and its barycenter is measured by the surrogate fGOT (sfGOT) distance, a tractable approximation of fGOT that handles graphs without known node correspondence, and is minimized end-to-end via backpropagation. We further provide a novel interpretation of sfGOT through the lens of the Hilbert-Schmidt Independence Criterion, showing that minimizing the sfGOT distance between two graphs is equivalent to maximizing statistical dependence between the spectral embedding of their nodes. Experiments on benchmark datasets demonstrate competitive performance over existing GDL methods on graph clustering and classification tasks.
\end{abstract}


\section{Introduction}
Learning with graph-structured data has attracted significant interests from the machine learning and signal processing communities, with applications in bioinformatics, cheminformatics, and social network analysis \cite{baxevanis2020bioinformatics, leach2007introduction, marin2011social}. Approaches to this problem typically rely on graph kernels, distances and latent space representations, applied in supervised, unsupervised or semi-supervised settings \cite{vishwanathan2010graph, nguyen2017semi, gilmer2017neural, nguyen2021learning,nguyen2019adaptive,nguyen2019recent,nguyen2023linear}.

In this work, we focus on dictionary learning (DL) for graphs, which is an unsupervised learning framework that learns a set of basis, called \emph{atoms}, constituting \emph{dictionary} \cite{tosic2011dictionary, kreutz2003dictionary}. Atoms are inferred by minimizing a reconstruction error when approximating each input as a weighted combination of atoms. The resulting combination coefficients serve as latent representations for downstream tasks such as classification and clustering. For graph-structured data, graph dictionary learning (GDL) frameworks based on Gromov-Wasserstein (GW) discrepancy \cite{memoli2011gromov} have been proposed to handle graphs of varying sizes and without node correspondence \cite{xu2020gromov, vincent2021online}. However, GW compares graphs via pairwise node distances (e.g., shortest-path distances), which primarily reflect local neighborhood structure rather than global patterns, thus failing to capture the global structure of graphs. %


Graph Optimal Transport (GOT) \cite{maretic2019got} offers an alternative by representing each graph as a zero-mean Gaussian distribution whose covariance is derived from the graph Laplacian pseudo-inverse, enabling comparison via the Wasserstein distance. This was extended to filter graph distances (fGOT) \cite{maretic2022fgot}, which introduces a spectral filter providing flexible spectral control over which structural scales are emphasized. To handle unknown node correspondence, \cite{maretic2022fgot} further introduced surrogate fGOT (sfGOT) as a tractable upper bound of fGOT.
Despite its practical utility, the structural meaning of sfGOT as a graph distance remains underexplored.

\textbf{Contributions.} (1) We provide a novel interpretation of sfGOT through the lens of Hilbert-Schmidt Independence Criterion (HSIC) \cite{gretton2005measuring}, showing that minimizing sfGOT between two graphs is equivalent to maximizing statistical dependence between the spectral embeddings of their nodes. (2) We propose \textbf{fGOT-GDL}, a GDL framework that uses fGOT as the barycenter comparison metric to capture global strucrural properties, and uses sfGOT as the tractable reconstruction loss for graphs without known node correspondence. (3) Experiments on benchmark datasets demonstrate competitive performance over existing GDL methods on clustering and classification.

\section{Background}
\textbf{Notation.} We represent a graph $G$ with $N$ nodes by its adjacency matrix $\mathbf{A} \in \mathbb{R}_+^{N \times N}$ and unnormalized Laplacian $\mathbf{L} = \mathbf{D} - \mathbf{A} \in \mathbb{R}^{N \times N}$, where $\mathbf{D} = \operatorname{diag}(\mathbf{A} \mathbf{1}_N)$ is the degree matrix. Since $\mathbf{A}$ and $\mathbf{L}$ are in one-to-one correspondence, we use them interchangeably. We denote by $\mathbf{I}_N \in \mathbb{R}^{N \times N}$ the identity matrix and $\mathbf{1}_N \in \mathbb{R}^N$ the all-ones vector. When $G$ has $d$-dimensional node attributes, we write $G = (\mathbf{A}, \mathbf{X})$ with $\mathbf{X} \in \mathbb{R}^{N \times d}$; otherwise $G = \mathbf{A}$.
Given a collection of observed graphs $\{\mathbf{A}_i \in \mathbb{R}_+^{N_i \times N_i}\}_{i=1}^I$, our goal is to learn a dictionary of $K$ atom graphs $\mathbf{U}_{1:K} = \{\mathbf{U}_k \in \mathbb{R}_+^{N_0 \times N_0}\}_{k=1}^K$, where all atoms share the same support size $N_0$, and to represent each observed graph $\mathbf{A}_i$ by an embedding vector $\boldsymbol{\lambda}_i \in \Delta^{K-1} = \{ \boldsymbol{\lambda} \in \mathbb{R}^K : \lambda_j \ge 0(\forall j), \sum_{j=1}^{K} \mathbf{\lambda}_j = 1 \}$. Note that observed graphs may differ in size ($N_i\neq N_j$ for $i\neq j$) and are not assumed to be aligned.

\subsection{Graph Dictionary Learning (GDL)}
Given the observed graphs $\{\mathbf{A}_i\}_{i=1}^I$, GDL jointly learns a dictionary $\mathbf{U}_{1:K}$, and embeddings $\boldsymbol{\lambda}_{1:I} = \{\boldsymbol{\lambda}_i\}_{i=1}^I$, such that each $\mathbf{A}_i$ is well approximated by a reconstruction $\widetilde{\mathbf{A}}_i$ built from $\mathbf{U}_{1:K}$ and $\lambda_i$:
\begin{equation}
    \min_{\{\mathbf{U}_{1:K}, \boldsymbol{\lambda}_{1:I}\} \in \Omega} \sum_{i=1}^I d_{\mathrm{loss}}^p\left(\widetilde{\mathbf{A}}_i, \mathbf{A}_i\right),
    \label{eq:gdl_obj}
\end{equation}
where $\Omega$ is the feasible constraint set, $d_{\mathrm{loss}}$ is a distance measure between the (possibly different-sized) reconstruction $\widetilde{\mathbf{A}}_i$ and observed graph $\mathbf{A}_i$, and $p$ is the order. The reconstruction is defined as the barycenter of the atoms weighted by $\boldsymbol{\lambda}_i$:
\begin{equation}
    \widetilde{\mathbf{A}}_i = b(\mathbf{U}_{1:K}, \boldsymbol{\lambda}_i) \triangleq \arg\min_{\mathbf{A} \in \Omega} \sum_{k=1}^K \lambda_{i,k} d_{\mathrm{b}}^q(\mathbf{A}, \mathbf{U}_k),
    \label{eq:barycenter_obj}
\end{equation}
where $d_{\mathrm{b}}$ is a distance between atom-sized graphs and $q$ is its order. When all graphs share the same size, setting $d_{\mathrm{b}}(\mathbf{A}, \mathbf{U}_k) = \|\mathbf{A} - \mathbf{U}_k\|_F$ (where $\lVert\cdot \rVert_{F}$ is the Frobenius norm), $q = 2$ reduces the barycenter to a linear combination $\widetilde{\mathbf{A}}_i = \sum_{k=1}^K \boldsymbol{\lambda}_{i,k} \mathbf{U}_k$, but this cannot handle graphs of varying sizes or unknown node correspondence. This is addressed in \cite{xu2020gromov, vincent2021online} by instantiating $d_{\mathrm{loss}}$ and $d_{\mathrm{b}}$ with the GW discrepancy \cite{memoli2011gromov}, yielding the GW Factorization (GWF) model, which aligns nodes to handle size mismatch but, as discussed in Section 1, relies on pairwise distances, failing to reflect global structural properties of graphs.

\subsection{Graph optimal transport (GOT) and its variants}
\textbf{GOT.} GOT \cite{maretic2019got} represents a graph $\mathbf{A} \in \mathbb{R}^{N \times N}$ with Laplacian $\mathbf{L}$ as a distribution over smooth graph signal $\mathbf{x} \in \mathbb{R}^N$ -- signals that vary slowly across connected nodes. Such smoothness is naturally encoded by a zero-mean Gaussian with the Laplacian pseudo-inverse as covariance, $\nu = \mathcal{N}(0, \mathbf{L}^\dagger)$, since low-frequency (smallest-eigenvalue) modes of $\mathbf{L}^\dagger$ dominate and encode global connectivity, while high-frequency modes capture local variation.
For two graphs of the same size with distributions $\nu_1 = \mathcal{N}(0, \mathbf{L}_1^\dagger)$ and $\nu_2 = \mathcal{N}(0, \mathbf{L}_2^\dagger)$, the GOT distance is computed by the Bures-Wasserstein (BW) distance \cite{bhatia2019bures}:
\begin{equation}
    \mathcal{W}_2^2(\nu_1, \nu_2)
    = \operatorname{tr}(\mathbf{L}_1^\dagger) + \operatorname{tr}(\mathbf{L}_2^\dagger)  
    - 2 \operatorname{tr}\left[ \left( \mathbf{L}_1^{\dagger/2} \mathbf{L}_2^\dagger \mathbf{L}_1^{\dagger/2} \right)^{1/2} \right].
\label{eq:got_form}
\end{equation}
\noindent
\textbf{fGOT.} To control which spectral scales are emphasized, fGOT \cite{maretic2022fgot} generalizes GOT via a \emph{graph filter} $g$, defined spectrally as $g(\mathbf{L}) = \mathbf{U} \hat{g}(\Lambda) \mathbf{U}^\top$ for $\mathbf{L} = \mathbf{U} \Lambda \mathbf{U}^\top$, where $\hat{g}: \mathbb{R}_{>0} \to \mathbb{R}_{>0}$ with $\hat{g}(0) = 0$ is the \emph{frequency response} applied elementwise to the eigenvalues. Passing while noise $\mathbf{w} \sim \mathcal{N}(0, \mathbf{I}_N)$ through the filter yields
$\mathbf{x}_f = g(\mathbf{L})\mathbf{w} \sim\mathcal{N}(0, g^2(\mathbf{L}))$. The fGOT distance is then computed by the BW distance between two graphs' filtered-signal distributions $\nu_{1,g} = \mathcal{N}(0, g^2(\mathbf{L}_1))$ and $\nu_{2,g} = \mathcal{N}(0, g^2(\mathbf{L}_2))$:
\begin{align}
\begin{split}
\label{eq:fgot_form}
    \mathcal{W}_2^2(\nu_{1,g}, \nu_{2,g}) 
    &= \operatorname{tr}(g^2(\mathbf{L}_1)) + \operatorname{tr}(g^2(\mathbf{L}_2)) \\
    &- 2 \operatorname{tr}\left[ \left( g(\mathbf{L}_1) g^2(\mathbf{L}_2) g(\mathbf{L}_1) \right)^{1/2} \right].
\end{split}
\end{align}

Common choices of filters include $g(\mathbf{L}) = \mathbf{L}^{\dagger/2}$, $g(\mathbf{L}) = \mathbf{L}^2$, and $g(\mathbf{L}) = \exp(-0.8\mathbf{L})$, each prioritizing a different spectral scale.\\

\noindent
\textbf{sfGOT.} Both distances in \eqref{eq:got_form} and \eqref{eq:fgot_form} assume known node correspondence. Under unknown alignment, we optimize over permutations $\mathbf{P} \in \mathcal{C}_{\mathrm{perm}}=\left\{ \mathbf{P} \in \{0, 1\}^{N \times N} : \mathbf{P} \mathbf{1}_N = \mathbf{1}_N, \, \mathbf{P}^\top \mathbf{1}_N = \mathbf{1}_N \right\}$, comparing $\nu_{1,g}$ to the permuted distribution $\nu_{2,g, \mathbf{P}} = \mathcal{N}(0, g^2(\mathbf{P} \mathbf{L}_2 \mathbf{P}^\top))$:
\begin{align}
\begin{split}
    &d_{\mathrm{fGOT},g}^2(\mathbf{A}_1, \mathbf{A}_2) = \min_{\mathbf{P} \in \mathcal{C}_{\mathrm{perm}}} \mathcal{W}_2^2(\nu_{1,g}, \nu_{2,g, \mathbf{P}}) \\
    &= \min_{\mathbf{P} \in \mathcal{C}_{\mathrm{perm}}} \Big[ \operatorname{tr}(g^2(\mathbf{L}_1)) + \operatorname{tr}(g^2(\mathbf{L}_2)) \\
   &- 2 \operatorname{tr}\left[ \left( g(\mathbf{L}_1) g^2(\mathbf{P} \mathbf{L}_2 \mathbf{P}^\top) g(\mathbf{L}_1) \right)^{1/2} \right] \Big].
\end{split}
    \label{eq:fgot_unaligned}
\end{align}
Since $\mathbf{P}$ appears under a matrix square root, solving \eqref{eq:fgot_unaligned} is computationally intractable. \cite{maretic2022fgot} instead replaces $\mathcal{W}_2^2$ with a tractable surrogate $\widetilde{\mathcal{W}}_2^2$, giving the surrogate fGOT (sfGOT) distance:
\begin{align}
\begin{split}
    &d_{\mathrm{sfGOT},g}^2(\mathbf{A}_1, \mathbf{A}_2) = \min_{\mathbf{P} \in \mathcal{C}_{\mathrm{perm}}} \widetilde{\mathcal{W}}_2^2(\nu_{1,g}, \nu_{2,g, \mathbf{P}}) \\
    &= \min_{\mathbf{P} \in \mathcal{C}_{\mathrm{perm}}} \Big[ \operatorname{tr}(g^2(\mathbf{L}_1)) + \operatorname{tr}(g^2(\mathbf{L}_2)) \\
    &- 2 \langle g(\mathbf{L}_1) \mathbf{P} g(\mathbf{L}_2), \mathbf{P} \rangle \Big].
\end{split}
    \label{eq:sfgot_form}
\end{align}
Lemma \ref{lem:sfgot_upper_bound} shows that $\widetilde{\mathcal{W}}_2^2$ is an upper bound of $\mathcal{W}_2^2$, implying that $d_{\mathrm{fGOT},g}^2$ \eqref{eq:fgot_unaligned} is upper bounded by $d_{\mathrm{sfGOT},g}^2$ \eqref{eq:sfgot_form}.

\begin{lemma}[\cite{maretic2022fgot}]
\label{lem:sfgot_upper_bound}
 For any graph filter $g(\cdot)$ and $\mathbf{P} \in \mathcal{C}_{\mathrm{perm}}$, $\mathcal{W}_2^2(\nu_{1,g}, \nu_{2,g,\mathbf{P}}) \le \widetilde{\mathcal{W}}_2^2(\nu_{1,g}, \nu_{2,g,\mathbf{P}})$; consequently, $d_{\mathrm{fGOT},g}^2\leq d_{\mathrm{sfGOT},g}^2$.
\end{lemma}

\section{Proposed Method}
We present two main contributions: a novel interpretation of the sfGOT distance \eqref{eq:sfgot_form} through the lens of the Hilbert-Schmidt independence criterion (HSIC) \cite{gretton2005measuring}, and \textbf{fGOT-GDL}, a GDL model using fGOT and sfGOT as the graph comparison metrics.

\subsection{HSIC-based interpretation of sfGOT}
As shown in Lemma~\ref{lem:sfgot_upper_bound}, the sfGOT distance is a tractable upper bound of fGOT, circumventing the matrix square root in \eqref{eq:fgot_unaligned}. However, its structural meaning as a graph distance remains underexplored. We now provide a new interpretation of sfGOT via HSIC.

\begin{definition}[HSIC \cite{gretton2005measuring}]
Let $\mathbf{x} \sim p$ and $\mathbf{z} \sim q$ be random variables on domains $\mathcal{X}$ and $\mathcal{Z}$, with RKHS kernels $k_1 : \mathcal{X} \times \mathcal{X} \to \mathbb{R}$ and $k_2 : \mathcal{Z} \times \mathcal{Z} \to \mathbb{R}$. HSIC is defined as the squared Hilbert-Schmidt norm of the cross-covariance operator $\mathcal{C}_{\mathbf{x}\mathbf{z}}$: $\mathrm{HSIC}(p, q) = \|\mathcal{C}_{\mathbf{x}\mathbf{z}}\|_{\mathrm{HS}}^2$. Given $N$ i.i.d. samples $\{(\mathbf{x}_i, \mathbf{z}_i)\}_{i=1}^N$, the empirical estimator is:
\begin{equation*}
    \widehat{\mathrm{HSIC}}(p, q) = \frac{1}{(N-1)^2} \operatorname{tr}(\mathbf{K}_1 \mathbf{H} \mathbf{K}_2 \mathbf{H}),
    \label{eq:hsic_form}
\end{equation*}
where $(\mathbf{K}_1)_{ab} = k_1(\mathbf{x}_a, \mathbf{x}_b)$, $(\mathbf{K}_2)_{cd} = k_2(\mathbf{z}_c, \mathbf{z}_d)$ are the Gram matrices and $\mathbf{H} = \mathbf{I}_N - \frac{1}{N}\mathbf{1}_N \mathbf{1}_N^\top$ is the centering matrix. $\mathrm{HSIC}(p, q) = 0$ iff $\mathbf{x}$ and $\mathbf{z}$ are independent (for universal kernels).
\end{definition}
\noindent
Rather than viewing a graph as a generator of smooth graph signals (as in GOT/fGOT), we treat each graph's $N$ nodes, embedded via the filtered spectrum, as samples of a random vector, and interpret $g(\mathbf{L}_1)$, $g(\mathbf{L}_2)$ as kernel Gram matrices over the nodes of $\mathbf{A}_1$ and $\mathbf{A}_2$, respectively. Since $\hat{g}(0) = 0$, the filtered Laplacian is centered, i.e., $\mathbf{H} g(\mathbf{L}) \mathbf{H} = g(\mathbf{L})$, so no additional centering is required. Using the identity $g(\mathbf{P} \mathbf{H}_2 \mathbf{P}^\top) = \mathbf{P} g(\mathbf{L}_2) \mathbf{P}^\top$ for $\mathbf{P} \in \mathcal{C}_{\mathrm{perm}}$ \cite{maretic2022fgot}, the empirical HSIC between $\mathbf{A}_1$ and the permuted $\mathbf{P} \mathbf{A}_2 \mathbf{P}^\top$ defines a similarity:
\begin{align}
    s(\mathbf{A}_1, \mathbf{A}_2) &= \max_{\mathbf{P} \in \mathcal{C}_{\mathrm{perm}}} (N-1)^2\widehat{\mathrm{HSIC}}\left(\mathbf{A}_1, \mathbf{P} \mathbf{A}_2 \mathbf{P}^\top\right) \nonumber \\
    &= \max_{\mathbf{P} \in \mathcal{C}_{\mathrm{perm}}} \operatorname{tr}\left( g(\mathbf{L}_1) \mathbf{P} g(\mathbf{L}_2) \mathbf{P}^\top \right),
    \label{eq:s_def}
\end{align}
i.e., $s(\mathbf{A}_1, \mathbf{A}_2)$ finds the node alignment that maximizes the statistical dependence between the spectral embeddings of $\mathbf{A}_1$'s nodes and $\mathbf{A}_2$'s nodes. Comparing \eqref{eq:s_def} to the sfGOT objective \eqref{eq:sfgot_form}, we obtain:
\begin{align*}
\begin{split}
    d^2_{\mathrm{sfGOT},g}(\mathbf{A}_1, \mathbf{A}_2) &=
    \operatorname{s}(\mathbf{A}_1, \mathbf{A}_1) + \operatorname{s}(\mathbf{A}_2, \mathbf{A}_2) - 2  s(\mathbf{A}_1, \mathbf{A}_2),
\end{split}
\end{align*}
where we used $g(\mathbf{P} \mathbf{L}_2 \mathbf{P}^\top) = \mathbf{P} g(\mathbf{L}_2) \mathbf{P}^\top$ from Lemma 1 of \cite{maretic2022fgot}, and $\mathbf{P}^\top\mathbf{P}=\mathbf{I}_N$.
So minimizing sfGOT between $\mathbf{A}_1$ and $\mathbf{A}_2$ is equivalent to maximizing $s(\mathbf{A}_1, \mathbf{A}_2)$--the dependence between these two graphs' filtered spectral embeddings. Proposition \ref{prop:eigenmode_decomposition} gives an eigenmode decomposition of $s(\mathbf{A}_1, \mathbf{A}_2)$.

\begin{proposition}
\label{prop:eigenmode_decomposition}
Let $g(\mathbf{L}_1) = \mathbf{U}\hat{g}(\Lambda) \mathbf{U}^\top$ and $g(\mathbf{L}_2) = \mathbf{V} \hat{g}(\Omega) \mathbf{V}^\top$ be eigendecompositions with orthonormal $\mathbf{U} = [\mathbf{u}_1, \dots, \mathbf{u}_N]$, $\mathbf{V} = [\mathbf{v}_1, \dots, \mathbf{v}_N]$ and eigenvalues $\Lambda = \operatorname{diag}(\lambda_1, \dots, \lambda_N)$, $\Omega = \operatorname{diag}(\omega_1, \dots, \omega_N)$. Then
\begin{equation*}
    s(\mathbf{A}_1, \mathbf{A}_2) = \max_{\mathbf{P} \in \mathcal{C}_{\mathrm{perm}}} \sum_{i,j=1}^N \hat{g}(\lambda_i) \hat{g}(\omega_j) \langle \mathbf{u}_i, \mathbf{P} \mathbf{v}_j \rangle^2.
    \label{eq:decomp_prop1}
\end{equation*}
\end{proposition}

\begin{proof}
By substituting the eigendecompositions and applying cyclicity of trace:
\begin{align*}
    \operatorname{tr}\left(g(\mathbf{L}_1) \mathbf{P} g(\mathbf{L}_2) \mathbf{P}^\top\right) &= \operatorname{tr}\left(\mathbf{U} \hat{g}(\Lambda) \mathbf{U}^\top \mathbf{P} \mathbf{V} \hat{g}(\Omega) \mathbf{V}^\top \mathbf{P}^\top\right) \\
    &=  \operatorname{tr}\left(\hat{g}(\Lambda) \mathbf{C} \hat{g}(\Omega) \mathbf{C}^\top\right),
\end{align*}
where $\mathbf{C} = \mathbf{U}^\top \mathbf{P} \mathbf{V}$, with $\mathbf{C}_{ij} = \mathbf{u}_i^\top \mathbf{P} \mathbf{v}_j = \langle \mathbf{u}_i, \mathbf{P} \mathbf{v}_j \rangle$. Expanding the trace over indices gives,
\begin{align*}
    \operatorname{tr}\left(\hat{g}(\Lambda) \mathbf{C} \hat{g}(\Omega) \mathbf{C}^\top\right)= \sum_{i,j=1}^N \hat{g}(\lambda_i) \hat{g}(\omega_j) \langle \mathbf{u}_i, \mathbf{P} \mathbf{v}_j \rangle^2.
\end{align*}
\end{proof}
\noindent
Proposition~\ref{prop:eigenmode_decomposition} shows that $s(\mathbf{A}_1, \mathbf{A}_2)$ measures how well the $i$-th eigenmode $\mathbf{u}_i$ of $\mathbf{A}_1$ aligns with the $j$-th permuted eigenmode $\mathbf{P} \mathbf{v}_j$ of $\mathbf{A}_2$, weighed by their spectral energies $\hat{g}(\lambda_i) \hat{g}(\omega_j)$. Furthermore, the filter $g$ controls which eigenmodes dominate this weighting: a low-pass filter emphasizes global structure (e.g., connectivity and community membership), while a high-pass filter emphasizes local variation (e.g., node degrees).

\subsection{GDL with filters and optimal transport}
We propose \textbf{fGOT-GDL}, which sets $d_{\mathrm{b}} = d_{\mathrm{fGOT},g}$, $d_{\mathrm{loss}} = d_{\mathrm{sfGOT},g}$, $p=2$, $q=2$. Training runs for $T$ epochs, each repeating three steps: (1) estimate each input graph's barycenter via fGOT; (2) compute the reconstruction loss via sfGOT; (3) update parameters via backpropagation to minimize the total reconstruction loss.

\begin{figure}[t]
\centering
\begin{subfigure}[b]{0.91\columnwidth}
    \centering
    \includegraphics[width=\textwidth]{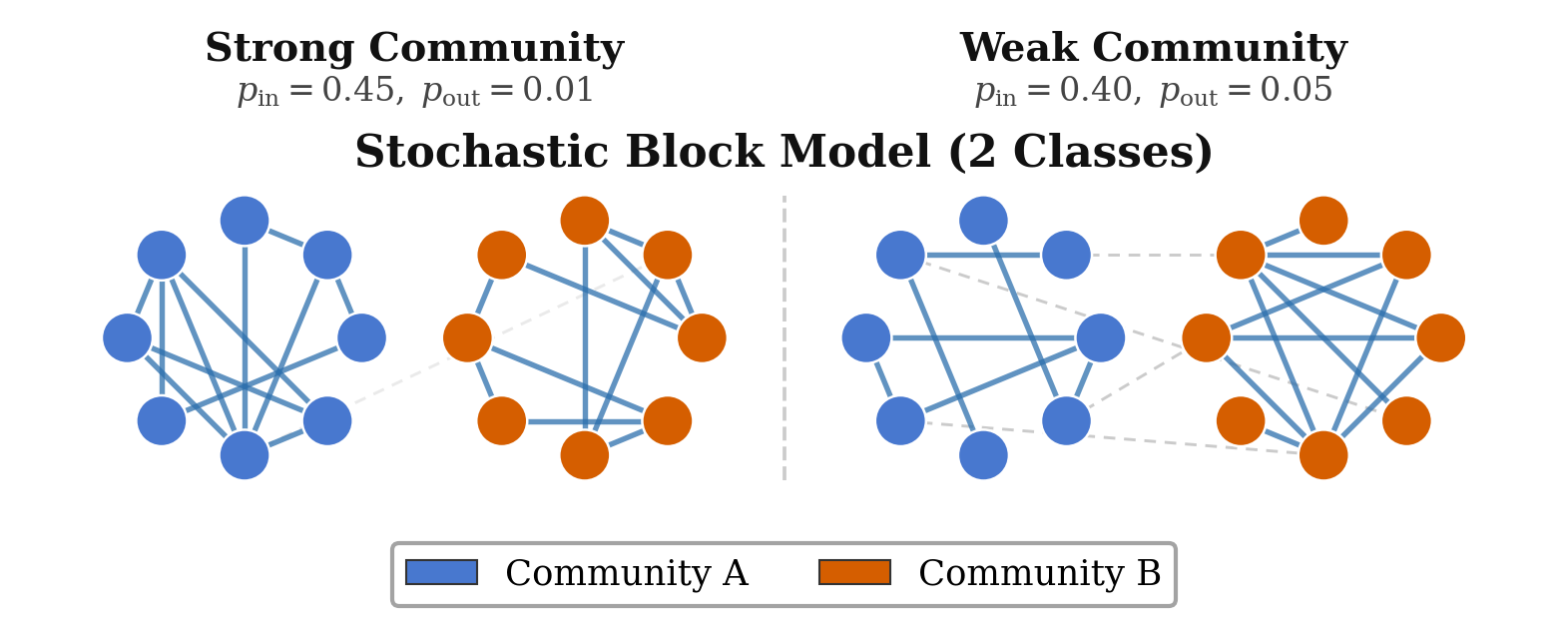}
    \caption{Synthetic SBM illustration}
\end{subfigure}
\hfill
\begin{subfigure}[b]{0.95\columnwidth}
    \centering
    \includegraphics[width=\textwidth]{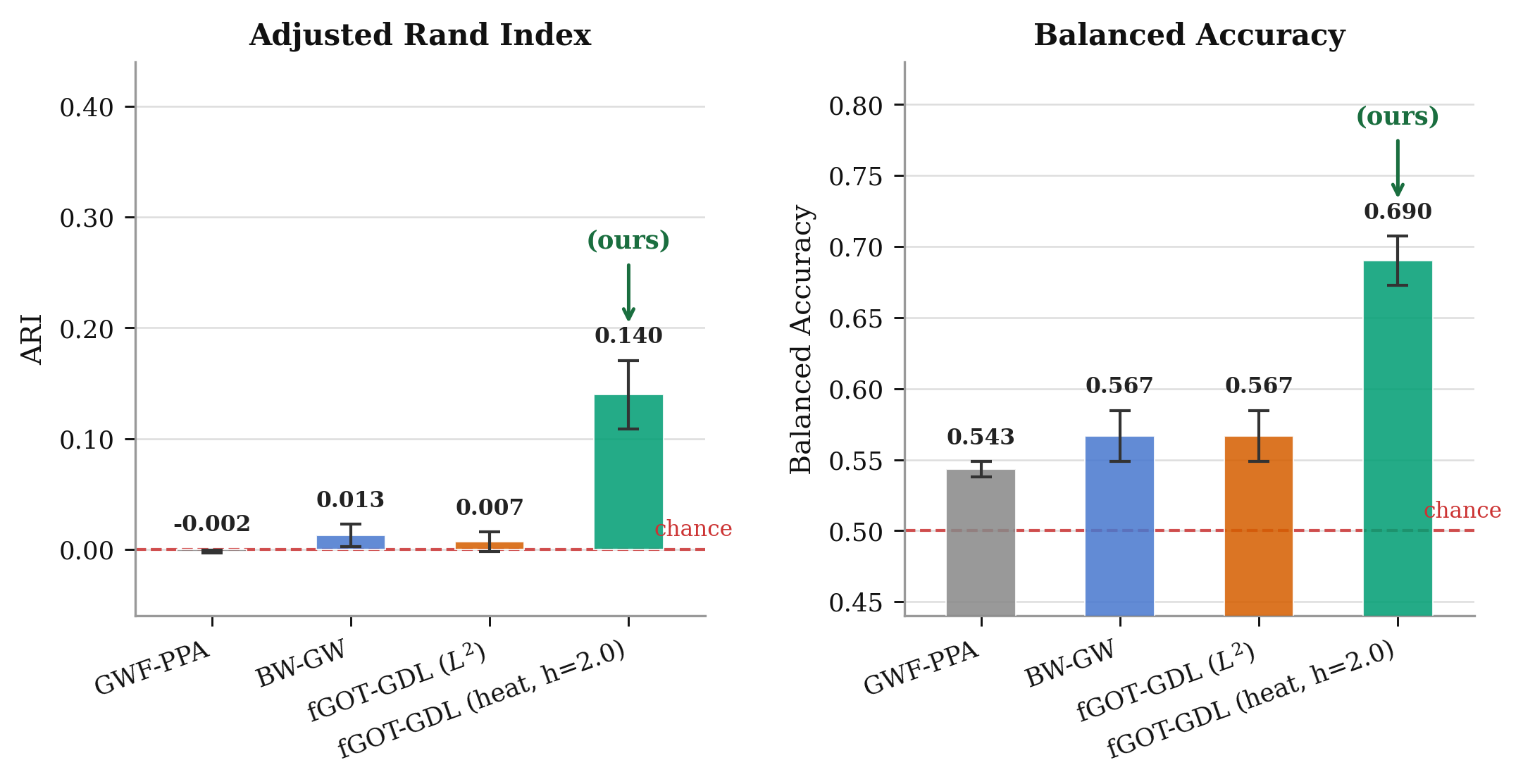}
    \caption{Recovery metrics}
\end{subfigure}
\caption{Comparison of GW-based GDLs and fGOT-GDL in recovering two-class SBM community structure.}
\label{fig:sbm}
\end{figure}
\noindent
\textbf{Barycenter via fGOT.} Setting $d_{\mathrm{b}} = d_{\mathrm{fGOT},g}$ in \eqref{eq:barycenter_obj} gives the Bures--Wasserstein Fr\'echet mean of $\{\nu_{k,g} = \mathcal{N}(0, g^2(\mathbf{L}_k))\}_{k=1}^K$, where $\mathbf{L}_k$ is the Laplacian of atom $\mathbf{U}_k$. This is estimated by the fixed-point iteration \cite{haasler2024bures}. For each $i=1,2,\dots,I$,
\begin{align}
    \mathbf{S}_{i,t+1} &\leftarrow \mathbf{S}_{i,t}^{-1/2} \left( \sum_{k=1}^K \lambda_{i,k} \left( \mathbf{S}_{i,t}^{1/2} \boldsymbol{\Sigma}_k \mathbf{S}_{i,t}^{1/2} \right)^{1/2} \right)^2 \mathbf{S}_{i,t}^{-1/2}, \nonumber \\
    \boldsymbol{\Sigma}_k &= g^2\left(\mathbf{L}_k + \frac{1}{N_0} \mathbf{I}_{N_0}\right).
    \label{eq:fixed_point}
\end{align}
After $n_b$ iterations, the barycenter Laplacian is $\widetilde{\mathbf{L}}_{i} = g^{-1}(\mathbf{S}_{i,n_b}^{1/2}) - \frac{1}{N_0} \mathbf{I}_{N_0}$, with adjacency $\widetilde{\mathbf{A}}_i = \operatorname{diag}(\widetilde{\mathbf{L}}_i) - \widetilde{\mathbf{L}}_i$.
\noindent
\begin{table}[t]
\caption{K-means clustering accuracy (\%).}
\label{tab:clustering}
\centering
\resizebox{\columnwidth}{!}{%
\begin{tabular}{lcccc}
\toprule
& \textbf{AIDS} & \textbf{PROTEIN} & \textbf{PROTEIN-F} & \textbf{IMDB-BINARY} \\
\midrule
\#Graphs & 2,000 & 1,113 & 1,113 & 1,000 \\
Avg. \#Nodes & 15.69 & 39.06 & 39.06 & 19.77 \\
Avg. \#Edges & 16.20 & 72.82 & 72.82 & 96.53 \\
Node attr. dim. & 4 & 1 & 29 & -- \\
\midrule
\multicolumn{5}{l}{\textit{Baselines from \cite{xu2020gromov}}} \\
FGWK & 91.0 $\pm$ 0.7 & 66.4 $\pm$ 0.8 & 66.0 $\pm$ 0.9 & 56.7 $\pm$ 1.5 \\
GWB-KM & 95.2 $\pm$ 0.9 & 64.7 $\pm$ 1.1 & 62.9 $\pm$ 1.3 & 53.5 $\pm$ 2.3 \\
GWF-BADMM & 97.6 $\pm$ 0.8 & 69.2 $\pm$ 1.0 & 68.1 $\pm$ 1.1 & 55.9 $\pm$ 1.8 \\
GWF-PPA & 99.5 $\pm$ 0.4 & 70.7 $\pm$ 0.7 & 69.3 $\pm$ 0.8 & 60.2 $\pm$ 1.6 \\
\midrule
\multicolumn{5}{l}{\textit{fGOT-GDL variants}} \\
Pure & 98.09 $\pm$ 0.24 & 67.10 $\pm$ 2.40 & 67.10 $\pm$ 2.40 & 63.36 $\pm$ 0.52 \\
Attr & 99.23 $\pm$ 0.29 & 71.18 $\pm$ 1.17 & 70.93 $\pm$ 1.13 & -- \\
MF & 99.49 $\pm$ 0.10 & 71.19 $\pm$ 0.78 & 71.18 $\pm$ 0.73 & 63.93 $\pm$ 0.69 \\
MFMS & \textbf{99.50 $\pm$ 0.05} & \textbf{71.84 $\pm$ 0.50} & \textbf{71.73 $\pm$ 0.51} & \textbf{64.00 $\pm$ 0.32} \\
\bottomrule
\end{tabular}%
}
\end{table}
\begin{table}[t]
\caption{Graph classification accuracy (\%)}
\label{tab:classification}
\centering
\begin{tabular}{lcc}
\toprule
\textbf{Methods} & \textbf{PROTEIN} & \textbf{IMDB-BINARY} \\
\midrule
\multicolumn{3}{l}{\textit{Baselines from \cite{xu2020gromov}}} \\
HOPPERK / GCK & 71.6 $\pm$ 3.7 & 56.9 $\pm$ 4.0 \\
PROPAK / SPK & 60.3 $\pm$ 5.1 & 56.2 $\pm$ 3.1 \\
FGWK & 75.1 $\pm$ 2.9 & 64.2 $\pm$ 3.3 \\
GWF-BADMM & 71.4 $\pm$ 3.6 & 62.4 $\pm$ 3.8 \\
GWF-PPA & 73.7 $\pm$ 2.0 & 63.9 $\pm$ 2.7 \\
\midrule
\multicolumn{3}{l}{\textit{fGOT-GDL variants}} \\
Pure & 70.53 $\pm$ 4.04 & 62.10 $\pm$ 3.90 \\
Attr & 70.08 $\pm$ 4.45 & -- \\
MF & 72.51 $\pm$ 4.18 & 62.80 $\pm$ 3.55 \\
MFMS & \textbf{74.31 $\pm$ 3.37} & \textbf{64.60 $\pm$ 3.20} \\
\bottomrule
\end{tabular}
\end{table}

\noindent
\textbf{Reconstruction loss via sfGOT.} Setting $d_{\mathrm{loss}} = d_{\mathrm{sfGOT},g}$ in \eqref{eq:gdl_obj}, the main difficulty is that the constraint $\mathbf{P}\in \mathcal{C}_\mathrm{perm}$ makes \eqref{eq:sfgot_form} a discrete optimization with factorially many feasible solutions. Following \cite{maretic2019got, maretic2022fgot}, we relax $\mathcal{C}_\mathrm{perm}$ to the convex set of possible soft assignments for each pair of $\mathbf{A}_i$ and $\widetilde{\mathbf{A}}_i$ of possibly different sizes $N_i$ and $N_0$:
\begin{align*}
    \mathcal{C}_{\mathrm{fuzzy}}=\left\{\mathbf{P}\in[0,1]^{N_i\times N_0}:\mathbf{P}\mathbf{1}_{N_0}=\frac{1}{N_i}\mathbf{1}_{N_i}, \mathbf{P}^\top\mathbf{1}_{N_i}=\frac{1}{N_0}\mathbf{1}_{N_0}\right\}.
\end{align*}
The sfGOT distance \eqref{eq:sfgot_form} is then approximated as:
\begin{align*}
    &d^2_{\mathrm{sfGOT},g}(\mathbf{A}_i, \widetilde{\mathbf{A}}_i) \approx\\ 
    & \operatorname{tr}(g^2(\mathbf{L}_i)) + \operatorname{tr}(g^2(\widetilde{\mathbf{L}}_i)) 
    - 2 \max_{\mathbf{P} \in \mathcal{C}_{\mathrm{fuzzy}}} \langle g(\mathbf{L}_i) \mathbf{P} g(\widetilde{\mathbf{L}}_i), \mathbf{P} \rangle.
    \label{eq:sfgot_soft}
\end{align*}

Since the first two terms are constant in $\mathbf{P}$, this reduces to the quadratic optimization problem:
$\max_{\mathbf{P} \in \mathcal{C}_{\mathrm{fuzzy}}} \langle g(\mathbf{L}_i) \mathbf{P} g(\widetilde{\mathbf{L}}_i), \mathbf{P} \rangle$, solved via the Sinkhorn-Knopp algorithm
 \cite{peyre2016gromov, knight2008sinkhorn} over $n_{\mathrm{out}}$ outer iterations. At outer iteration $m$ ($m\leq n_{\mathrm{out}}$), we fix one $\mathbf{P}=\mathbf{P}^{(m)}$, and solve the following entropy-regularized problem by $n_{\mathrm{in}}$ inner Sinkhorn iterations \cite{peyre2016gromov, cuturi2013sinkhorn} :
\begin{equation}
    \max_{\mathbf{P} \in \mathcal{C}_{\mathrm{fuzzy}}} \langle g(\mathbf{L}_i) \mathbf{P}^{(m)} g(\widetilde{\mathbf{L}}_i), \mathbf{P} \rangle + \varepsilon \mathcal{H}(\mathbf{P}).
    \label{eq:entropic_step}
\end{equation}
The final plan $\mathbf{P}^{(n_{\mathrm{out}})}$ is then used to evaluate $d_{\mathrm{sfGOT},g}$. 

\begin{figure}[t]
\centering
\begin{subfigure}[b]{0.32\columnwidth}
    \centering
    \includegraphics[width=\textwidth]{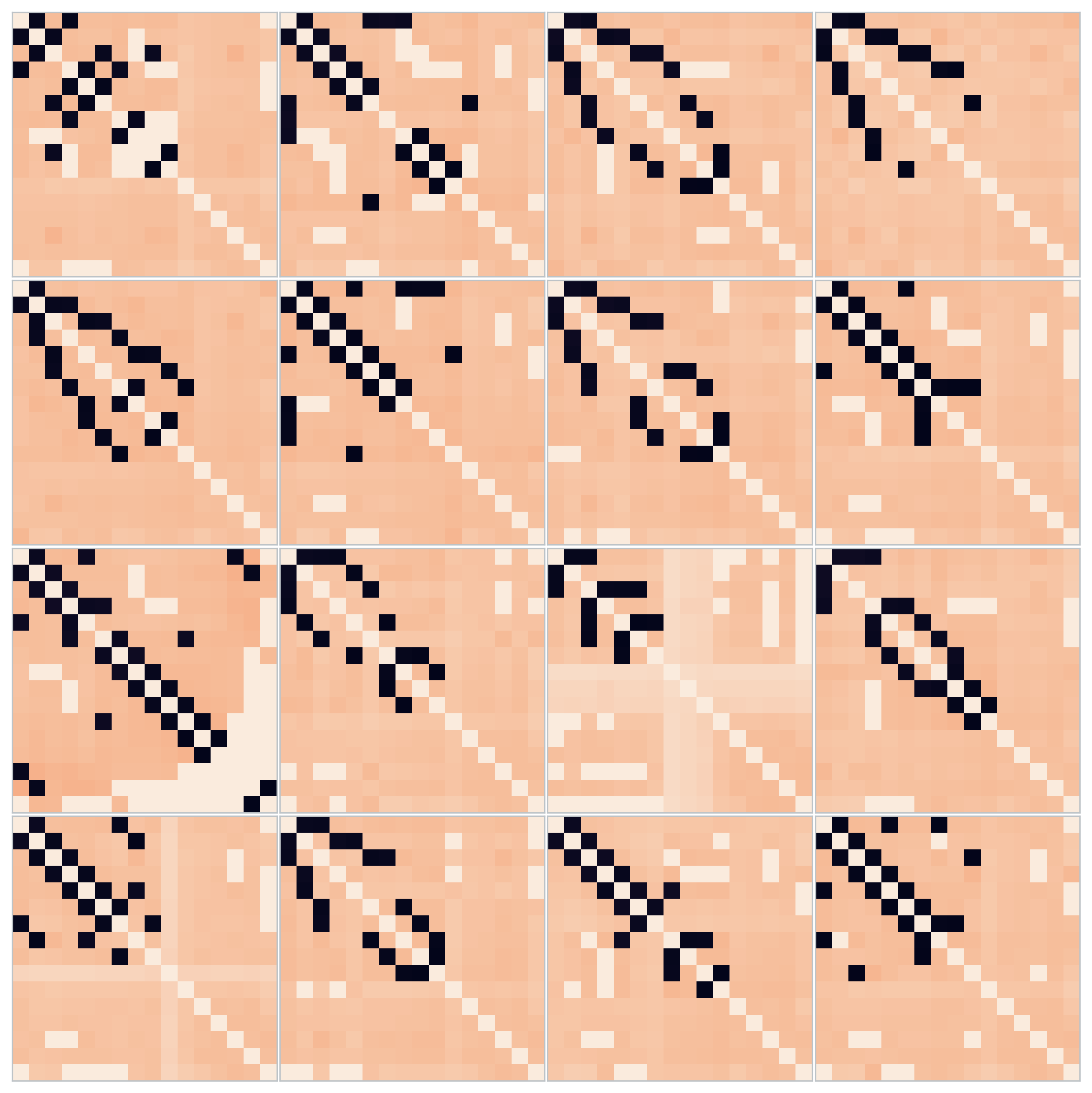}
    \caption{AIDS}
\end{subfigure}
\hfill
\begin{subfigure}[b]{0.32\columnwidth}
    \centering
    \includegraphics[width=\textwidth]{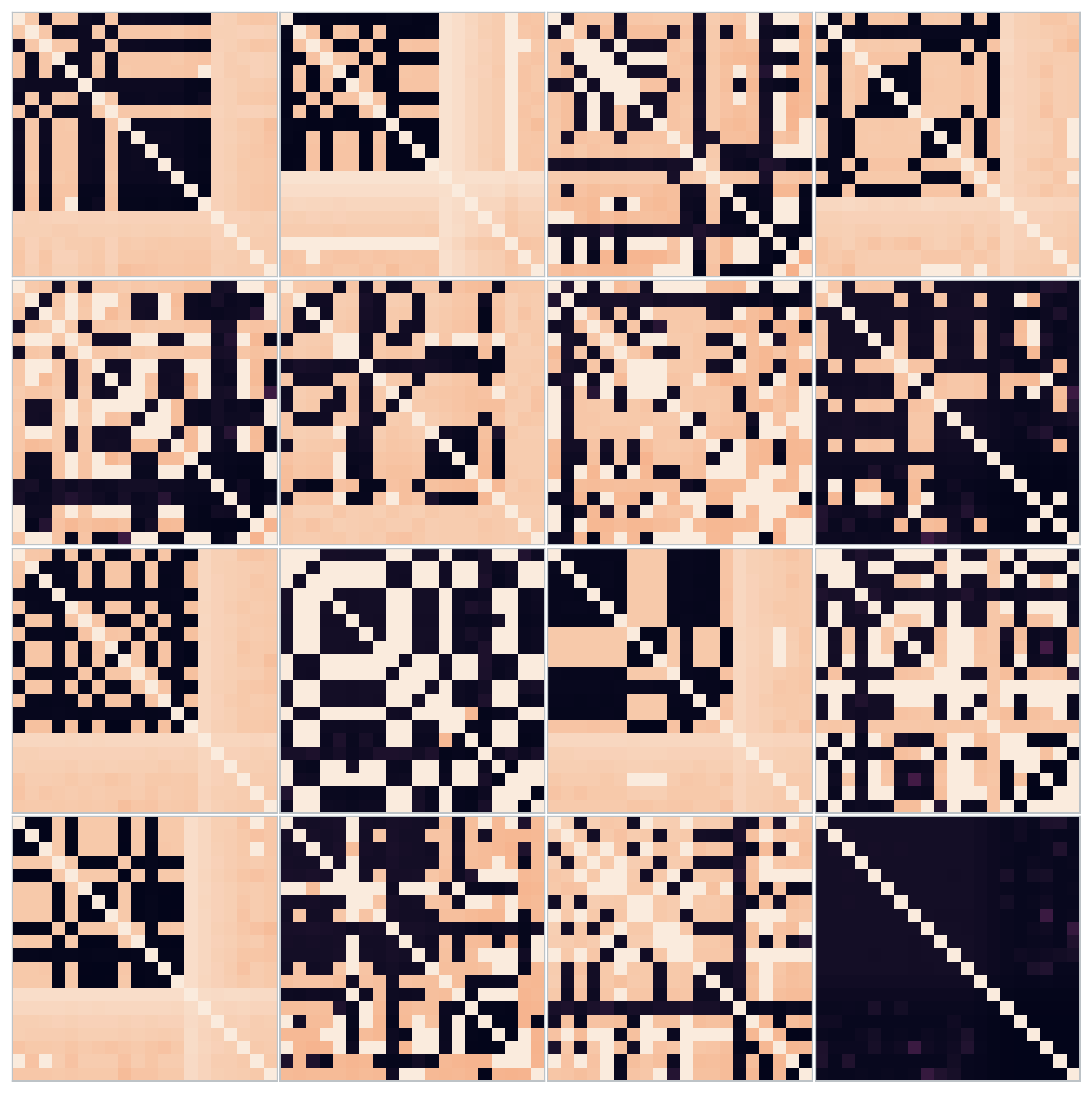}
    \caption{IMDB-BINARY}
\end{subfigure}
\hfill
\begin{subfigure}[b]{0.32\columnwidth}
    \centering
    \includegraphics[width=\textwidth]{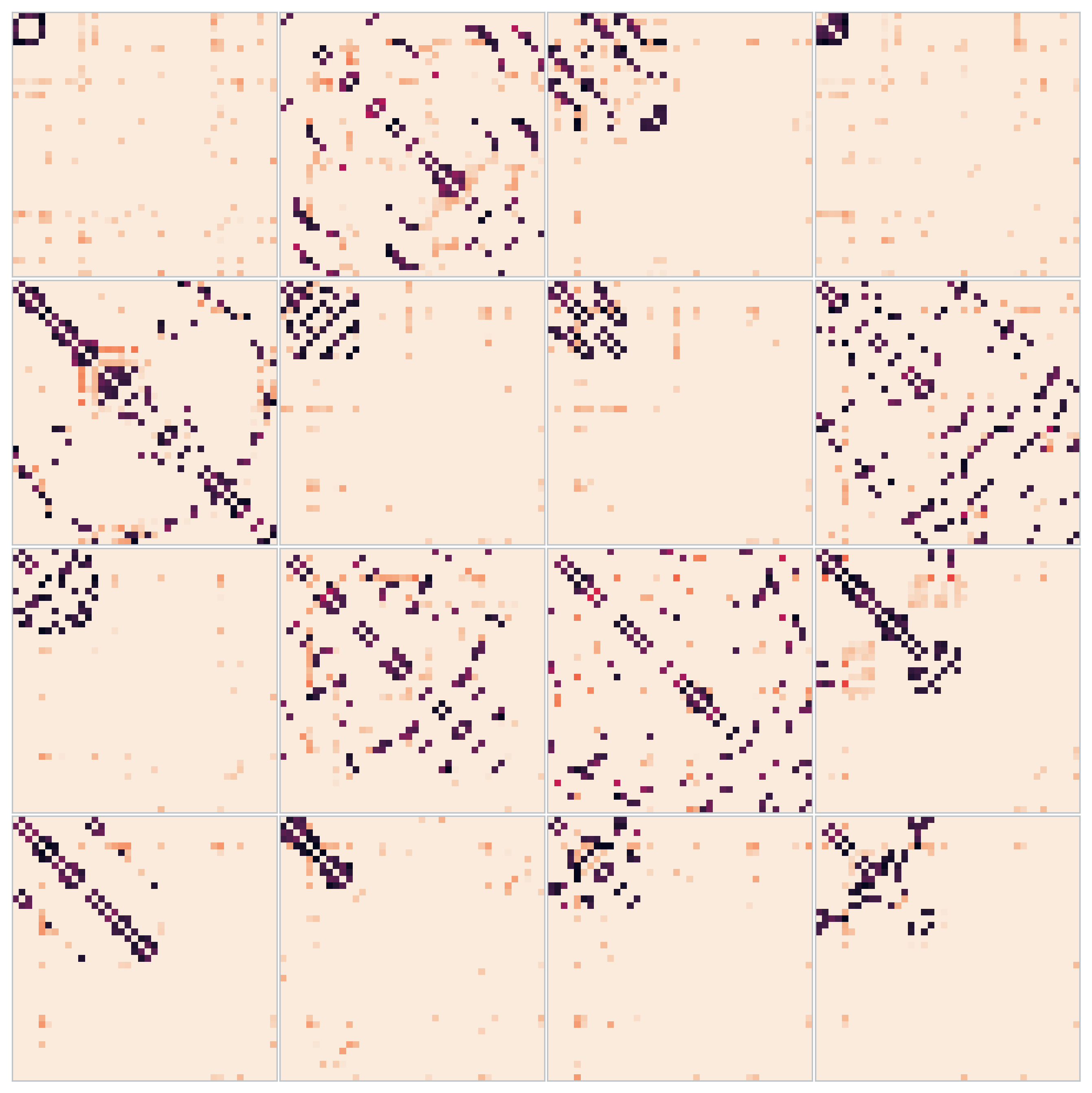}
    \caption{PROTEIN}
\end{subfigure}
\caption{Visualization of the adjacency matrices of atoms learned by fGOT-GDL, using three datasets: AIDS, IMDB-BINARY, PROTEIN.}
\label{fig:atom_heatmaps}
\end{figure}

\noindent
\textbf{Reparameterization.} Solving \eqref{eq:gdl_obj} is challenging due to the constraints on $\mathbf{U}_{1:K}$ and $\boldsymbol{\lambda}_{1:I}$. To handle these constraints, we reformulate \eqref{eq:gdl_obj} as an unconstrained optimization problem via
\begin{align*}
    \mathbf{U}_k &= \Phi(\boldsymbol{\theta}_k) = \operatorname{diag}_0\left( \frac{\operatorname{softplus}(\boldsymbol{\theta}_k) + \operatorname{softplus}(\boldsymbol{\theta}_k)^\top}{2} \right), \\
    \boldsymbol{\lambda}_i &= \operatorname{softmax}(\mathbf{z}_i), 
\end{align*}
where $\boldsymbol{\theta}_k \in \mathbb{R}^{N_0 \times N_0}$ and $\mathbf{z}_i \in \mathbb{R}^K$ are unconstrained parameters; $\operatorname{diag}_0(\cdot)$ zeroes the diagonal elements; $\Phi(\boldsymbol{\theta}_k)$ ensures $\mathbf{U}_k$ is nonnegative, symmetric, zero-diagonal; $\operatorname{softmax}(\boldsymbol{z}_i)$ ensures $\boldsymbol{\lambda}_i \in \Delta^{K-1}$. This yields an unconstrained optimization problem,
$
    \min_{\boldsymbol{\theta}_{1:K}, \boldsymbol{z}_{1:I}} \sum_{i=1}^I d^2_{\mathrm{sfGOT},g}\left(\widetilde{\mathbf{A}}_i, \mathbf{A}_i\right),
$
where $\widetilde{\mathbf{A}}_i$ is estimated by minimizing $\sum_{k=1}^K \boldsymbol{\lambda}_{i,k} d_{\mathrm{fGOT},g}^2(\mathbf{A}, \Phi(\boldsymbol{\theta}_k))$,
and $\boldsymbol{\lambda}_i = \operatorname{softmax}(\mathbf{z}_i)$. The parameters are updated end-to-end via backpropagation. 

\section{Experiments}

\subsection{Synthetic experiments}
We evaluate whether the learned embedding can distinguish two community-structure regimes using stochastic block model (SBM) graphs (Fig.~\ref{fig:sbm}): \textbf{Strong community} ($p_{\mathrm{in}} = 0.45, p_{\mathrm{out}} = 0.01$, well-separated) and \textbf{Weak community} ($p_{\mathrm{in}} = 0.40, p_{\mathrm{out}} = 0.05$, subtler). Each class has 25 graphs of 32 nodes (50 total), over 3 data-generation and 2 training seeds (6 runs). Embeddings $\boldsymbol{\lambda}_i$ are clustered via K-means (two clusters and 10 random initializations) and scored by Adjusted Rand Index (ARI) and balanced accuracy.

\noindent
\textbf{Baselines.} \textbf{GWF-PPA} \cite{xu2020gromov} learns atoms and $\boldsymbol{\lambda}_i$ via GW discrepancy for both barycenter and reconstruction loss, solved with the Proximal Point Algorithm ($K=2$, $N_0=32$, $\varepsilon = 0.1$, five PPA layers, one-dimensional node embedding, Adam optimizer with learning rate of 0.1, and $T=16$ training epochs). \textbf{BW-GW} uses fGOT for the barycenter (as in fGOT-GDL) but GW discrepancy for reconstruction loss, isolating the effect of sfGOT vs. GW as the loss.

\noindent
\textbf{fGOT-GDL setup.} We use $K=2$, $N_0 = 32$, $T=16$ training epochs, $n_b = 10$ barycenter iterations, learning rate of 0.05, $n_{\mathrm{out}} = 5$, $n_{\mathrm{in}} = 20$ , and $\varepsilon = 0.1$. We evaluate two filter choices: a high-pass filter ($\hat{g}(\lambda)=\lambda^2$), and a low-pass heat kernel filter ($\hat{g}(\lambda)=e^{-2\lambda}$).

\noindent
\textbf{Results.} The low-pass filter substantially outperforms the high-pass filter (ARI 0.140 vs. 0.007) and both baselines (ARI -- 0.002 GWF-PPA, 0.013 BW-GW; balanced accuracy 0.690 vs. 0.543/0.567). Distinguishing community strength requires sensitivity to global topology, which GW's local pairwise distances fail to capture, while fGOT-GDL's low-frequency eigenmodes naturally detect it.


\subsection{Real-world experiments}
We evaluate on AIDS, PROTEIN, PROTEIN-F, and IMDB-BINARY (statistics in Table~\ref{tab:clustering}) using four fGOT-GDL variants: \textbf{Pure} (structural Laplacian, single heat kernel filter), \textbf{Attr} (attribute-mixed Laplacian $\mathbf{L}_{\alpha}=(1-\alpha)\mathbf{L}+\alpha \mathbf{L}_{\mathrm{attr}}$, where $\mathbf{L}_{\mathrm{attr}} = \operatorname{diag}(\mathbf{S}_{\mathrm{attr}}\mathbf{1}_N) - \mathbf{S}_{\mathrm{attr}}$ with $\mathbf{S}_{\mathrm{attr}} = \bar{\mathbf{X}}\bar{\mathbf{X}}^\top$ and $\bar{\mathbf{X}}$ the row-wise $\ell_2$-normalized $\mathbf{X}$, $\alpha$ tuned via Bayesian optimization), \textbf{MF} (concatenated embeddings across 3 heat-kernel bandwidths), and \textbf{MFMS} (MF across two random seeds).


\noindent
\textbf{Setup.} We use the heat kernel filter $\hat{g}(\lambda) = \exp(-h\lambda)$ throughout. For single-filter experiments, the bandwidth is $h = 0.18$ (AIDS), $h = 0.30$ (IMDB-BINARY), and $h = 0.25$ (PROTEIN, PROTEIN-F). For multi-filter experiments, three bandwidths are used: $h \in \{0.12, 0.18, 0.25\}$ (AIDS), $h \in \{0.25, 0.30, 0.40\}$ (IMDB-BINARY), and $h \in \{0.15, 0.22, 0.25\}$ (PROTEIN, PROTEIN-F). The number of atoms is $K = 64$ with support sizes $N_0 = 16$ (AIDS), $N_0 = 20$ (IMDB-BINARY), and $N_0 = 40$ (PROTEIN, PROTEIN-F). Dictionary learning uses learning rate of 0.003, $n_b = 5$ fGOT barycenter iterations, $n_{\mathrm{out}} = 30$ outer iterations, and $n_{\mathrm{in}} = 50$ inner iterations with $\varepsilon = 0.1$ for computing the sfGOT distance. Atom graphs are initialized by randomly selecting $K$ observed graphs, cropped to $N_0$ nodes by degree-based ranking if larger, or zero-padded otherwise. Results are reported over 10 random seeds.

\noindent
\textbf{Evaluation.} For clustering, K-means with two clusters and 20 random initializations is applied to $\boldsymbol{\lambda}_i$; accuracy is reported after matching cluster labels to ground-truth classes by the best label permutation. We compare against FGWK \cite{titouan2019optimal}, GWB-KM \cite{peyre2016gromov}, and GWF-BADMM/GWF-PPA \cite{xu2020gromov}. For classification, an RBF kernel matrix is constructed from $\boldsymbol{\lambda}_i$ and a kernel SVM is trained under stratified 10-fold cross-validation. We compare against SPK \cite{borgwardt2005shortest}, HOPPERK \cite{feragen2013scalable}, PROPAK \cite{neumann2016propagation}, GCK \cite{shervashidze2009efficient}, FGWK, GWF-BADMM, and GWF-PPA \cite{xu2020gromov}.

\noindent
\textbf{Results.} \textbf{Pure} already matches GW-based baselines on AIDS (98.09\%) and IMDB-BINARY (63.36\%); \textbf{Attr} improves further when node attributes are informative. \textbf{MF} consistently outperforms \textbf{Pure}, confirming multi-bandwidth filters capture complementary spectral information; \textbf{MFMS} reduces initialization variance and achieves the best overall results: 99.50\% (AIDS), 71.84\%/71.73\% (PROTEIN/-F), 64.00\% (IMDB-BINARY) clustering accuracy (Table~\ref{tab:clustering}), and 74.31\%/64.60\% classification accuracy (Table~\ref{tab:classification}) --- outperforming GWF-BADMM/PPA and competitive with the strongest baseline, FGWK (75.1\%/64.2\%). Fig.~\ref{fig:atom_heatmaps} visualizes learned atom graphs as adjacency heatmaps, offering qualitative insight into recovered structural prototypes.

\section{Conclusions}
We proposed fGOT-GDL, a graph dictionary learning framework using fGOT as the barycenter metric and sfGOT as the reconstruction loss to capture global graph structure. We further showed that minimizing sfGOT is equivalent to maximizing HSIC-based statistical dependence between spectral embeddings, with experiments confirming competitive performance on clustering and classification.

\vspace{-2mm}
\section*{Acknowledgment}
This work was supported by MEXT KAKENHI Grant Number 26K21294.
\vspace{-1mm}
\bibliographystyle{IEEEbib}
\bibliography{references}

\end{document}